%% file: main.tex
\newcommand{\MJ}[1]{{\color{red} MJ: #1}} 
\newcommand{\RC}[1]{{\color{blue} RC: #1}} 

\documentclass[letterpaper, 10 pt, conference]{ieeeconf}  

\IEEEoverridecommandlockouts                              
\renewcommand{\baselinestretch}{0.98}

\usepackage[font={footnotesize},tablename=Table]{caption}
\usepackage{comment}
\usepackage{color}
\usepackage{cite}
\usepackage{float}

\usepackage{graphics} 
\usepackage{epsfig} 
\usepackage{mathptmx} 
\usepackage{times} 
\usepackage{amsmath} 
\usepackage{amssymb}  
\usepackage{tabularx}
\newtheorem{theorem}{\textbf{\large{Theorem}}}
\newtheorem{assumption}{\textbf{Assumption}}
\newtheorem{lemma}[theorem]{\textbf{Lemma}}
\newtheorem{definition}{\textbf{Definition}}
\newtheorem{remark}[theorem]{\textbf{Remark}}
\newcommand{\oprocendsymbol}{\hbox{$\bullet$}}
\newcommand{\oprocend}{\relax\ifmmode\else\unskip\hfill\fi\oprocendsymbol}

\usepackage{array}
{\renewcommand{\arraystretch}{1.4}}%
\newcolumntype{P}[1]{>{\centering\arraybackslash}p{#1}}

\newcommand{\red}[1]{{\color{red} #1}}
\newcommand{\prl}[1]{\left(#1\right)} 
\newcommand{\brl}[1]{\left[#1\right]}

\DeclareMathOperator{\tr}{\mathbf{tr}}

\title{\LARGE \bf
Safe Multi-Agent Interaction through Robust Control Barrier Functions with Learned Uncertainties
}

\input{sym.tex}

\author{Richard Cheng, Mohammad Javad Khojasteh, Aaron D. Ames, and Joel W. Burdick
\thanks{R.~Cheng, A.~D.~Ames, and J.~W.~Burdick are with the
Department of Mechanical and Civil Engineering,
California Institute of
    Technology.
M.~J.~Khojasteh is with
    the Department of Electrical Engineering, California Institute of
    Technology.
        { \{rcheng,mjkhojas,ames,jwb\}@caltech.edu}}%
}

\begin{document}
\maketitle
\begin{abstract}
Robots operating in real world settings must navigate and maintain safety while interacting with many heterogeneous agents and obstacles. Multi-Agent Control Barrier Functions (CBF) have emerged as a computationally efficient tool to guarantee safety in multi-agent environments, but they assume perfect knowledge of both the robot dynamics and other agents' dynamics. While knowledge of the robot's dynamics might be reasonably well known, the heterogeneity of agents in real-world environments means there will \textit{always} be considerable uncertainty in our prediction of other agents' dynamics. This work aims to learn high-confidence bounds for these dynamic uncertainties using Matrix-Variate Gaussian Process models, and incorporates them into a robust multi-agent CBF framework. We transform the resulting min-max robust CBF into a quadratic program, which can be efficiently solved in real time. We verify via simulation results that the nominal multi-agent CBF is often violated during agent interactions, whereas our robust formulation maintains safety with a much higher probability and adapts to learned uncertainties.
\end{abstract}

\section{INTRODUCTION}
Collision-free robot navigation in natural multi-agent environments is vital for a myriad of robotic applications, such as self-driving cars, navigation in crowds, etc. However, placing robots in rapidly evolving, uncertain environments introduces many challenges in guaranteeing safety~\cite{Safe_mbrl_nips17,fisac2018general,learning_safe_mpc,wabersich2018safe,janson2018safe}. Uncertainty in the prediction of other agents' trajectories is inevitable, and robots should learn and account for this uncertainty to ensure safe operation. 

In this work, we utilize the Multi-Agent Control Barrier Function (CBF) proposed in \cite{Borrmann2015} to generate  low-level controllers that guarantee collision-free behavior. The Multi-Agent CBF uses an optimization-based controller to  prevent  the  robot  from  entering  unsafe sets (i.e. states leading to inevitable collision) in a \textit{minimally invasive} fashion. However, previous work either assumes that the robot dynamics and the other agents' dynamics are perfectly known, or considers highly conservative worst-case bounds \cite{Wang2016,Singletary2020}. It is clear though that all natural environments (e.g. settings with humans) are fraught with varying levels of uncertainty. For example, human trajectories remain notoriously difficult to predict and are highly stochastic \cite{Bartoli2018}. Thus, both capturing and accounting for uncertainty is crucial for safe navigation. 
   \begin{figure}[ht]
        \begin{center}            
        \includegraphics[width=0.47\textwidth]{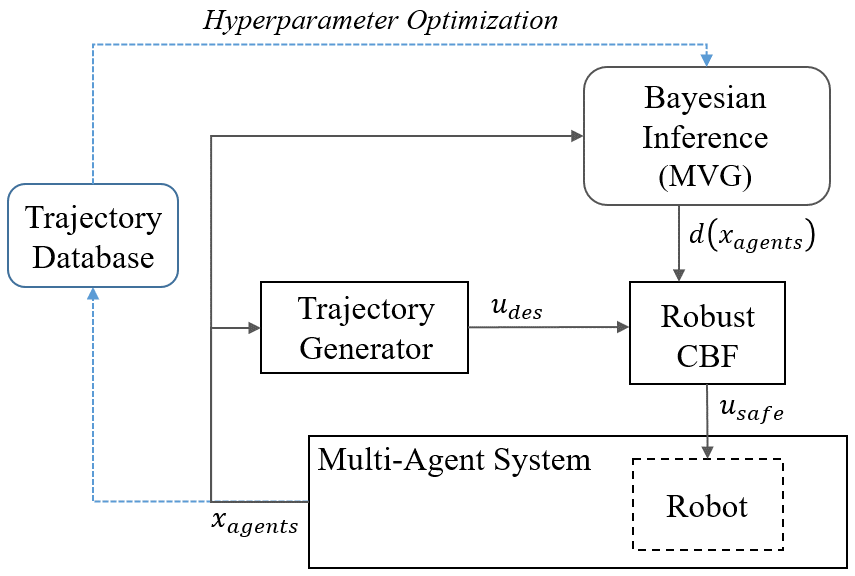}
        \end{center}
        \caption{~Diagram overviewing the control structure. Our approach guarantees safety by utilizing a Bayesian Inference Module to learn dynamic uncertainties, and handles them with our proposed Robust CBF module.}
        \label{fig:block_diagram}
    \end{figure}

    \begin{figure*}[t]
        \begin{center}            
        \includegraphics[width=0.93\textwidth]
        {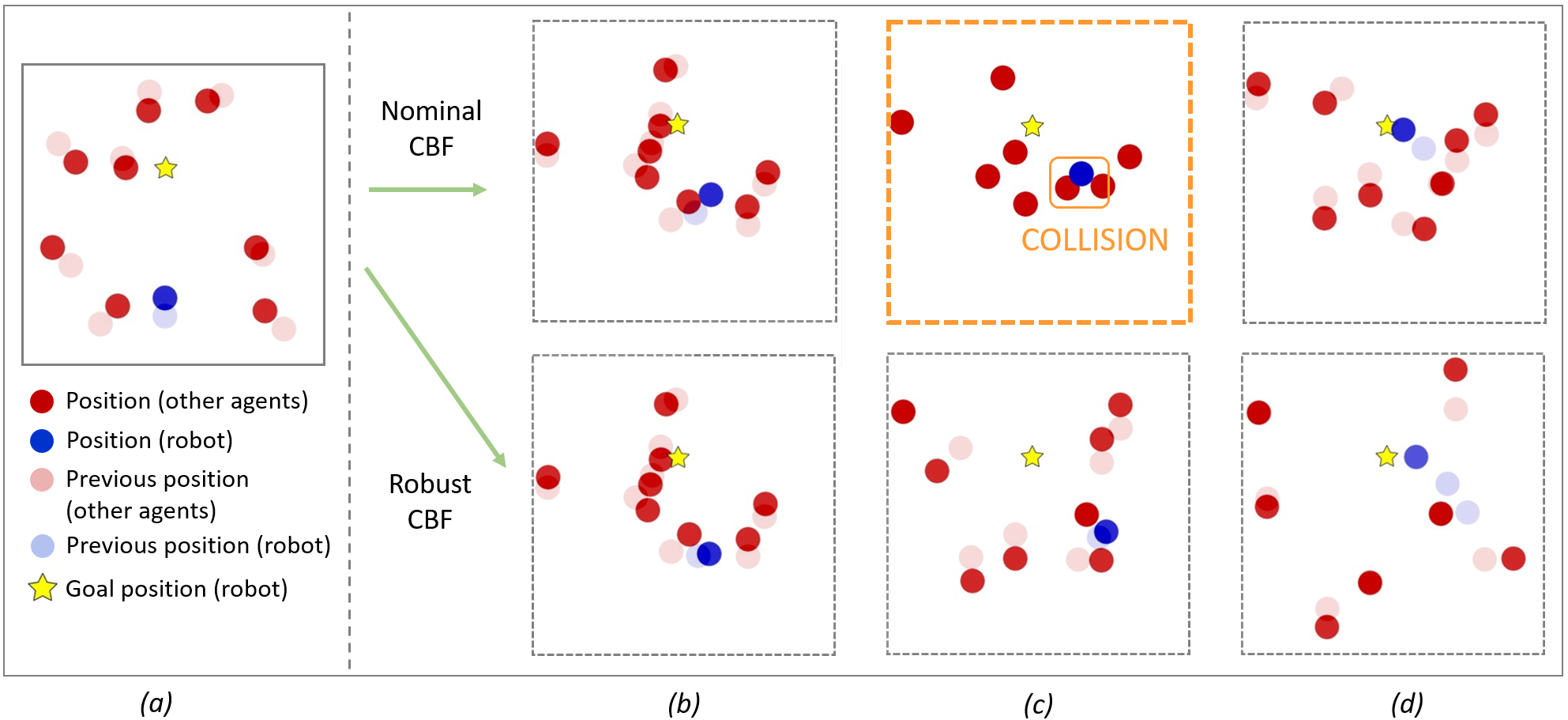}
        \end{center}
        \caption{~Sample path of a multi-agent system based on the nominal CBF (cf.~\cite{Borrmann2015}) and our proposed Robust CBF. The robot (blue) tries to navigate from a start position to random goal position while avoiding collisions with other agents (red). Approximately half of the other agents blindly travel towards their own randomly chosen goal, while the rest exhibit varying degrees of collision-avoidance behavior (the robot does not know their behavior apriori). For more details and results of simulations, see~Sec.~\ref{sec:simulnavig}. \textbf{(a)} Initial robot/environment configuration, \textbf{(b)}  Intermediate configuration, \textbf{(c)} Intermediate configuration showing that the nominal CBF controller experiences collision \textit{(top)}, while the robust CBF avoids collision \textit{(bottom)}. \textbf{(d)} Final configuration before robot reaches its goal position (star). See \texttt{https://youtu.be/hXg5kZO86Lw} for the simulation videos.}
        \label{fig:markov_chain_fig}
    \end{figure*}
\noindent
\textbf{Statement of contributions:} The goal  of  this  paper  is  to (1) learn uncertainty bounds \textit{online} from  agents' observed trajectories, and (2) incorporate those uncertainty bounds into the Multi-Agent CBF while maintaining computational efficiency of the underlying controller (i.e. a quadratic program). 

Our approach focuses first on learning high-confidence polytopic bounds on the, \textit{possibly coupled}, uncertainties in both the robot dynamics and other agents' dynamics. To achieve this, we utilize Matrix-Variate Gaussian Processes (MVG) and optimize their hyperparameters \textit{offline} from interaction data; this allows us to predict ellipsoidal uncertainty in our dynamics \textit{online}, which we convert to an uncertainty polytope given a desired confidence level. Using these polytopic bounds, we formulate a robust CBF as a min-max optimization problem over the robot controls and the potential uncertainties, respectively. We then transform this min-max problem into a quadratic program that can be efficiently solved to find a safe control action that is robust with respect to our estimated uncertainty. See Fig.~\ref{fig:block_diagram} for an overview of our approach.

\smallskip
\noindent
\textbf{Organization:} Sec.~\ref{sec:2} goes over related work in the safe collision avoidance literature, Sec.~\ref{sec:3} provides background information on CBFs used to guarantee safety, and Matrix-Variate Gaussian Processes used to estimate correlated uncertainties. Sect.~\ref{secrefiVw} introduces the robust multi-agent CBF and shows how it can be efficiently solved under given uncertainty bounds. Sec.~\ref{sec:5} looks at how to learn these confidence bounds over the uncertainty. Finally, Sec.~\ref{sec:6} presents simulation results illustrating the benefits of our algorithm and verifying the safety of our controller. 

\section{RELATED WORK}
\label{sec:2}

Multi-agent collision avoidance has been a long-studied problem with different approaches proposed for enabling safe control in varying situations. Velocity obstacles is a popular approach that involves limiting control actions to a set of ``safe'' actions, though its constant velocity with linear dynamics assumption is limiting~\cite{Fiorini1998,VanBerg2008a}. Related works in this direction have loosened these assumptions, but require significant sampling of the action space and do not incorporate dynamic uncertainty~\cite{Wilkie2009,Bareiss2015}. More recently, Buffered Voronoi Cells (BVC) have been proposed as a tool to provide safety guarantees with only positional information, though safety guarantees are provided only under linear dynamics and without uncertainty~\cite{Zhou2017,Wang2018}. Also,~\cite{Fisac2015,Chen2018,Bansal2018} provide safety guarantees, under worst case disturbances, by solving the  Hamilton-Jacobi-Isaacs equation to obtain a minimally invasive control law. However, the heavy computational expense prohibits applicability to large-scale multi-agent systems. 

Reinforcement learning methods have emerged recently, which directly learn actions in multi-agent settings in response to the observed environment \cite{Chen2017,Everett2018}. However, these methods provide no formal guarantees of safety, and as such are prone to collision in novel environments. Furthermore, they have not been shown to scale to settings with many agents. To ensure safety in the reinforcement learning settings, the work~\cite{2019arXiv191012639Z} combined a safe backup policy with the learned policies. \cite{Cheng2019} incorporated safety constraints into the reinforcement learning framework, although this work considered only \textit{decoupled} uncertainties and was limited to polytopic safety constraints. Bayesian inference was utilized in~\cite{wang2018safe,khojasteh2019probabilistic,fan2019bayesian} to learn system dynamics in an online manner while ensuring safety (with high probability) using Control Barrier Functions for a single agent.

Control Barrier Functions (CBF) are a tool for enforcing set invariance of dynamical systems \cite{Ames2014,Ames2017,nguyen2016exponential}, and they have been used to guarantee collision avoidance in multiagent settings by projecting desired actions, to the closest (in least-squares sense) safe actions according to a CBF condition \cite{clark2019control,yang2019self,srinivasan2019weighted}. However, defining a valid CBF for general systems remains a challenge, especially in cases with significant uncertainty. Recent works have looked at learning CBFs for general systems, either implicity or explicitly \cite{Gurriet2019,Srinivasan2020}. A multi-agent CBF was defined explicitly for multi-agent systems in the case of continuous-time linear dynamics \cite{Borrmann2015,Wang2016}. However, this multi-agent CBF does not incorporate uncertainty or nonlinearity in the dynamics. Other work on robust CBFs deals only with highly conservative worst-case bounds \cite{Singletary2020}. We seek to address this gap in this paper.

Our work builds upon the current literature by computing the minimally invasive action necessary to maintain safety in a computationally efficient manner \textit{for discrete-time systems while incorporating nonlinear dynamics and learned uncertainty}.

\medskip
\section{BACKGROUND}
\label{sec:3}
Our robotic system is represented by nonlinear control-affine dynamics in discrete-time:
\begin{equation}
\begin{split}
& x_{t+1}
=
\begin{bmatrix}
p_{t+1}  \\
v_{t+1}  \\
z_{t+1} \\
\end{bmatrix}
=
\underbrace{
\begin{bmatrix}
f_p(x_t)  \\
f_v(x_t)  \\
f_z(x_t) \\
\end{bmatrix}
}_{f(x_t)}
+
\underbrace{
\begin{bmatrix}
g_p(x_t)  \\
g_v(x_t)  \\
g_z(x_t) \\
\end{bmatrix}
}_{g(x_t)}
u
+
\underbrace{
\begin{bmatrix}
d_p(x_t)  \\
d_v(x_t)  \\
d_z(x_t)  \\
\end{bmatrix}
}_{d(x_t)}
, \\
\end{split}
\label{eq:transition_dynamics}
\end{equation}
\noindent
where $p \in \mathbb{R}^2$, $v \in \mathbb{R}^2$, and $z \in \mathbb{R}^{n-4}$ denote position, velocity and other states, respectively. Here, $f_{j}$, $g_{j}$, and $d_{j}$ are real-valued functions, for $j \in \{p,v,z\}$, and $u \in \mathcal{U}$ where $\mathcal{U} := \{u \in \mathbb{R}^m : \| u \|_2 \leq u_{max} \} $. 
The functions $f(x)$ and $g(x)$ are assumed to be \textit{known}, whereas $d(x)$ represents \textit{unknown} uncertainty in the dynamics, which we model with a Gaussian process [cf.~Sec.~\ref{sec:323}]. 
We assume that this system has relative degree 2 with respect to the positional output $p$; in discrete time, this directly implies that $g_p(x) = 0_{2 \times 2}$.
Similarly, let us represent the other agents within our multi-agent system with dynamics,
 \begin{equation}
 x^{(i)}_{t+1}
 =
 \begin{bmatrix}
 p^{(i)}_{t+1}  \\
 v^{(i)}_{t+1}  \\
 z^{(i)}_{t+1}  \\
 \end{bmatrix}
 =
 \underbrace{
 \begin{bmatrix}
 f^{(i)}_p(x_t)  \\
 f^{(i)}_v(x_t)  \\
 f^{(i)}_z(x_t)  \\
 \end{bmatrix}
 }_{f^{(i)}(x_t)}
 +
 \underbrace{
 \begin{bmatrix}
 d^{(i)}_p(x_t)  \\
 d^{(i)}_v(x_t)  \\
 d^{(i)}_z(x_t)  \\
 \end{bmatrix}
  }_{d^{(i)}(x_t)}
 ,
 \label{eq:dynamics_agents}
 \end{equation}
\noindent
where $i \in \mathbb{N}$ indexes each of the other agents in our system. We assume the control input for other agents are a (uncertain) function of their state at the given time, so we do not show control inputs explicitly in~\eqref{eq:dynamics_agents}.

As our robot is interacting with other unknown agents, it will be important for us to account for the uncertainties, $d, d^{(i)}$, when considering safety via CBFs. For the rest of the paper, we assume that we perfectly observe each agent's current state $x_t$, but do not know (i.e. can only estimate) their uncertain dynamics $d$ and  $d^{(i)}$. 

\subsection{Control Barrier Functions}

Consider a safe set, $\mathcal{C}$, defined by the super-level set of a continuously differentiable
function $h: \mathbb{R}^n \rightarrow \mathbb{R}$:
\begin{equation}
\begin{split}
&  \mathcal{C} := \{x \in \mathbb{R}^n: h(x) \geq 0 \}.
\end{split}
\label{eq:safe_set}
\end{equation}
To maintain safety during the learning process, the system state must always remain within the safe set $\mathcal{C}$ (i.e. the set $\mathcal{C}$ is \textit{forward invariant} with respect to the system dynamics). A set $\mathcal{C}$ is forward invariant if for every $x_0 \in \mathcal{C}$, $x(t) \in \mathcal{C}$ for all $t \geq 0$. In our multi-agent setting, this set could include all states where collision can be avoided given the robot's input bounds. \textit{Control barrier functions} utilize a Lyapunov-like argument to provide a sufficient condition for ensuring forward invariance of the safe set $\mathcal{C}$ under controlled dynamics.
\smallskip
\begin{definition}
	Given a set $\mathcal{C} \in \mathbb{R}^n$ defined by (\ref{eq:safe_set}), the continuously differentiable function $h: \mathbb{R}^n \rightarrow \mathbb{R}$ is a \textit{discrete-time control barrier function} (CBF) for dynamical system (\ref{eq:transition_dynamics}) if there exists $\eta \in [0,1]$  such that for all $x_t \in \mathcal{C}$, 
	\begin{equation}
	\sup_{u_t \in \mathcal{U}} \mbox{CBC}(x_t, u_t) 
	\geq 0 ,
	\label{eq:discrete_CBF}
	\end{equation}
\end{definition} 
\noindent
for $CBC(x_t, u_t):=h(x_{t+1}(u_t)) + (\eta - 1) h_t(x_t)$, where CBC stands for control barrier condition.  

\smallskip
If a function $h(x)$ is a CBF, then there exists a controller such that the set $\mathcal{C}$ is forward invariant \cite{Agrawal2017,Ames2017}. In other words, system safety is guaranteed by ensuring satisfaction of condition (\ref{eq:discrete_CBF}). Our goal is to ensure safety, by computing a minimally invasive control action that satisfies (\ref{eq:discrete_CBF}), in an online fashion. In particular, we utilize the following multi-agent CBF inspired by \cite{Borrmann2015}:
\begin{equation}
\begin{split}
& h(x) := \frac{\Delta p^T \Delta v}{\| \Delta p\|} + \sqrt{a_{max} (\| \Delta p\| - D_s)} ~ , \\
\end{split}
\label{eq:MA_CBF}
\end{equation}

\noindent
where $a_{max}$ represents our robot's max acceleration in the collision direction, $D_s$ is the collision margin, $\Delta p = p - p^{(i)}$ is the positional difference between the agents, and $\Delta v = v - v^{(i)}$ is the velocity difference between the agents. 

The work~\cite{Borrmann2015} introduced a CBF similar to~\eqref{eq:MA_CBF} for \textit{continuous-time linear} systems, such that $a_{max}$ can be determined easily. In Sec.~\ref{secrefiVw}, under an appropriate assumption, we show that \eqref{eq:MA_CBF} is a valid CBF for the discrete-time nonlinear dynamics \eqref{eq:transition_dynamics} and \eqref{eq:dynamics_agents}. 

\subsection{Matrix-Variate Gaussian Process Regression}
\label{sec:323}
Here we illustrate how Bayesian learning can be used to acquire a distribution over the uncertainty in the dynamics. Since we are estimating a multivariate uncertainty $d(x)$, we must consider potential correlations in its components. Thus, 
 we use the Matrix-Variate Gaussian Process (MVG) model to learn the system dynamics and uncertainty from data. By learning $\mu_d(x)$ and $\Sigma_d(x)$ in tandem with the controller, we can obtain high probability confidence intervals on the unknown dynamics, which adapt/shrink as we obtain more information (i.e. measurements) on the system. We first start by defining the MVG distribution~\cite{khojasteh2019probabilistic,gupta2018matrix,louizos2016structured,StructuredPBP} as follows.
 \smallskip
\begin{definition}
We say the random matrix $\bfX \in \mathbb{R}^{N \times n}$ is distributed according to a MVG distribution when its probability density function is defined as:
\begin{align}
\label{eq:MVGP}
\begin{aligned}
p(\bfX; \bfM, \Sigma, \Omega) := \frac{\exp\prl{ -\frac{1}{2}\tr\brl{\Omega^{-1}(\bfX-\bfM)^T \Sigma^{-1}(\bfX-\bfM)}}}{(2\pi)^{Nn/2}\det(\Sigma)^{n/2}\det(\Omega)^{N/2}},
\end{aligned}
\end{align}
where $\bfM \in \mathbb{R}^{N \times n}$ denotes the mean, and  $\Sigma \in \mathbb{R}^{N \times N}$ encodes the covariance matrix of the rows, and $\Omega \in \mathbb{R}^{n \times n}$ encodes the covariance matrix of the columns.
In this case, we write $\bfX \sim \calM\calN(\bfM,\Sigma,\Omega)$, and we have $\textit{vec}(X) \sim \mathcal{N}(\textit{vec}(\bfM), \Sigma  \otimes \Omega)$, where $\textit{vec}(\bfX) \in \mathbb{R}^{Nn}$ is the vectorization of $\bfX$, obtained by stacking the columns of $X$, and $\otimes$ is the Kronecker product. 
\end{definition}

\smallskip
We continue by modeling $d(x)$ as a MVG on $\mathbb{R}^n$. Without loss of generality, we assume zero mean for the MVG with positive semi-definite parameter covariance matrix $\Omega \in \mathbb{R}^{m\times m}$, and kernel $\kappa:\mathbb{R}^n \times \mathbb{R}^n \rightarrow \mathbb{R}$. That is,
\begin{align*}
    \textit{vec}(d(x_1), \ldots, d(x_N))\sim \mathcal{N}(\textbf{0},~\Sigma(x) \otimes \Omega),
\end{align*}
where $\Sigma \in \mathbb{R}^{N \times N}$ with $\Sigma_{i,j}=\kappa(x_i,x_j)$. There are many potential choices for the kernel function $\kappa(x_i, x_j)$ (cf.~\cite{williams2006gaussian}), though we utilize the simple squared-exponential kernel in this work,
\begin{equation}
    \kappa(x_i, x_j) = \sigma^2 \exp{\Big(\frac{-\|x_i - x_j\|^2}{2l^2} \Big)} ~ ,
\end{equation}
\noindent
where $\sigma$ and $l$ are kernel hyperparameters. Since $d(x)$ is an MVG, the training observations $y_{[N]}:=[d(x_1), \ldots, d(x_N)]^{T}$ at sampling points $x_{[N]}:=[x_1, \ldots,x_N]$, and the predictive target $d(x_*)$ at query test point, $x_*$, are jointly Gaussian as 
\begin{align*}
    \begin{bmatrix}
 y_{[N]}  \\
  d(x_*) \\
 \end{bmatrix}
 \sim \calM\calN\left(\textbf{0},\begin{bmatrix}
 K(x_{[N]},x_{[N]})  &  K(x_{*},x_{[N]}) \\
   K(x_{*},x_{[N]}) & \kappa(x_{*},x_{*}) \\
 \end{bmatrix},\Omega \right),
\end{align*}
where $K(x_{[N]},x_{[N]}) \in \mathbb{R}^{N \times N}$ with $[K(x_{[N]},x_{[N]})]_{i,j}=\kappa(x_i,x_j)$, and $K(x_{*},x_{[N]}) \in \mathbb{R}^{1 \times N}$ with $[K(x_{*},x_{[N]})]_{i}=\kappa(x_*,x_i)$.  
Thus, we can compute the posterior distribution as follows:
\begin{equation}
\begin{split}
    & \textit{vec}(d(x_*)) \sim \mathcal{N} \big( \hat{M} ~ , ~ \hat{\Sigma} \otimes \hat{\Omega} \big) \\
    & ~~~~~~~~\hat{M} = K(x_*, x_{[N]})^T K(x_{[N]}, x_{[N]})^{-1} y_{[N]} \\
    & ~~~~~~~~\hat{\Sigma} = \kappa(x_*, x_*) - K(x_*, x_{[N]})^T  K(x_{[N]}, x_{[N]})^{-1} K(x_*, x_{[N]}) \\
    & ~~~~~~~~\hat{\Omega} = \Omega
    \label{eq:GP_inf}
\end{split}
\end{equation}
\noindent
This allows us to estimate our unknown dynamics and their possibly correlated uncertainties.

\medskip
\section{ROBUST MULTI-AGENT CBF}
\label{secrefiVw}
In this section, we first show that under certain assumptions,~\eqref{eq:MA_CBF} is a multi-agent CBF for our discrete-time nonlinear dynamics. Then, given bounds on the uncertainty in each agents' dynamics, we incorporate robustness to these uncertainties into the CBF while maintaining the computational efficiency of a quadratic program. 

~

\noindent
\textbf{Extending Multi-Agent CBF to Discrete-Time, Nonlinear Systems:} The multi-agent CBF introduced in \cite{Borrmann2015} was originally designed for \textit{continuous-time linear} systems. However, we prove that under proper assumption, $h(x)$ defined in \eqref{eq:MA_CBF} is a discrete-time CBF for the \textit{discrete-time nonlinear} system \eqref{eq:transition_dynamics}/\eqref{eq:dynamics_agents}. The tradeoff is the additional conservativeness in $a_{max}$ introduced by the following assumption. Intuitively, this assumption ensures that the robot can accelerate in any direction relative to the other agents, as proved in Lemma \ref{lemma:accel}. 
\medskip
\begin{assumption}
Assume that for all $x \in \mathcal{C}$, $g_v(x)$ is invertible and $\frac{\|\beta_v(x)\|}{\sigma_{min}(g_v(x))  u_{max}} < 1$, where $\beta_v(x) = f_v(x) + d_v(x) - f_v^{(i)}(x) - d_v^{(i)}(x) - \Delta v_t$ and $\sigma_{min}(g_v(x))$ is the minimum singular value of $g_v(x)$. 
\label{assumption:act}
\end{assumption}
\medskip
\begin{remark}
This assumption ensures controllability and places restrictions on our agent's dynamics with relation to its actuator authority. If $u_{max}$ is large, the restriction is minimal/non-existent, and vice-versa. As a simple example, a car at rest would \textit{not} satisfy this assumption, though a moving car \textit{would} likely satisfy this assumption (with a higher velocity corresponding to larger $a_{max}$). \oprocend
\end{remark}
\medskip
\begin{lemma}
\textit{Under Assumption \ref{assumption:act}, which places controllability restrictions on the dynamics, the expression (\ref{eq:MA_CBF}), defining set $\mathcal{C}$, represents a discrete-time CBF for system (\ref{eq:transition_dynamics}), with} 
\begin{equation}
a_{max} = \min_{x} \Big[ \sigma_{min}(g_v(x)) u_{max} - \| \beta_v(x) \| \Big] > 0.
\end{equation}
\label{lemma:accel}
\end{lemma}
\begin{proof}
First, we must show that set $\mathcal{C}$ defined by expression (\ref{eq:MA_CBF}) is control invariant for the dynamics (\ref{eq:transition_dynamics}), given that the robot has acceleration authority in any direction of at least $a_{max}$ for all $x \in \mathcal{C}$. For this, we rely on the same proof structure in \cite{Borrmann2015}, with the main difference being that we have discrete-time (rather than continuous-time) dynamics. Let $\Delta \hat{v}(x_t)$ denote the component of velocity $v(x_t)$ \textit{in the direction of collision}.
\begin{equation}
    \Delta \hat{v}(x_t) = \frac{\Delta p ^T \Delta v}{\| \Delta p \|}
\end{equation}
We know that collision can be avoided if we can match the other agent's velocity (i.e. $\Delta \hat{v} = 0$) by the time we reach them. If we assume that we can accelerate by $a_{max}$ in any direction, we are guaranteed that we can achieve $\Delta \hat{v} = 0$ within time $T_c = \frac{-\Delta \hat{v}(x_t)}{a_{max}}$. In our discrete-time formulation, the following condition implies collision avoidance:
\begin{equation}
\begin{split}
    & \Delta \hat{v}(x_t) T_c + \| \Delta p \| \geq D_s, \\
    &  \Big( \frac{\Delta p ^T \Delta v}{\| \Delta p \|} \Big)^2 \leq a_{max} ( \| \Delta p \| - D_s).
\end{split}
\end{equation}
Note that this constraint is only active when two agents are moving closer to each other ($\Delta \hat{v} < 0$), and no constraint is needed when two agents are moving away from each other ($\Delta \hat{v} \geq 0$). Therefore, collision can always be avoided under the following condition,
\begin{equation}
    - \frac{\Delta p ^T \Delta v}{\| \Delta p \|}  \leq \sqrt{a_{max} ( \| \Delta p \| - D_s)}.
\end{equation}

Based on our geometric argument, we know that the set $\mathcal{C} :=  \{x \in \mathbb{R}^n: h(x) \geq 0 \}$ defined by $h := \frac{\Delta p ^T \Delta v}{\| \Delta p \|} + \sqrt{a_{max} ( \| \Delta p \| - D_s)}$ is control invariant. This implies that $h(x)$ is a discrete-time CBF \cite{Ames2017}, \textit{given that the robot can accelerate by at least $a_{max}$ in any direction}.

Therefore, our second step is to show that for all $x \in \mathcal{C}$ and any unit vector $\hat{e}$, it holds that $\sup_{u \in \mathcal{U}} ~ \| \big( ( v_{t+1}(x, u) - v^{(i)}_{t+1}(x) )  - ( v_t - v^{(i)}_t \big)^T \hat{e} \| \geq a_{max} > 0$. 
\begin{equation}
\begin{split}
& \sup_{u \in \mathcal{U}} ~ \| \big( ( v_{t+1}(x, u) - v^{(i)}_{t+1}(x) )  - ( v_t - v^{(i)}_t \big)^T \hat{e} \| \\
& ~~~~~~~~~~~ = \sup_{u \in \mathcal{U}} \| \big( \beta_v(x) + g_v(x) u \big)^T \hat{e} \| \\
& ~~~~~~~~~~~ = \sup_{u  \in \mathcal{U}} \| \hat{e}^T \beta_v(x) + \hat{e}^T g_v(x) u \| \\
& ~~~~~~~~~~~ \geq \sup_{u \in  \mathcal{U}} \| \hat{e}^T g_v(x) u \| - \| \hat{e}^T \beta_v(x) \| \\
& ~~~~~~~~~~~ \geq \sigma_{min}(g_v(x)) u_{max} - \| \beta_v(x) \| \\
& ~~~~~~~~~~~ \geq \min_{x} \Big[ \sigma_{min}(g_v(x)) u_{max} - \| \beta_v(x) \| \Big] = a_{max} \\
& ~~~~~~~~~~~ > 0 ~,
\end{split}
\end{equation}
where the last inequality follows directly from Assumption \ref{assumption:act}. Therefore, we are guaranteed that the robot can accelerate by at least $a_{max} > 0$ in any direction. Combined with the first part of the proof, this shows that the set $\mathcal{C}$ defined by \eqref{eq:MA_CBF} is a discrete-time CBF.
\end{proof}
\medskip
\noindent
\textbf{Incorporating Robustness into CBF:} While uncertainty in robot/environmental dynamics can be directly incorporated into the Control Barrier Condition (CBC) \textit{for simple systems/constraints} (e.g. linear CBFs) \cite{Cheng2019}, this is not the case for the multi-agent CBF with discrete-time dynamics. Unfortunately, uncertainty cannot be directly incorporated into the CBC while maintaining a quadratic program. 

Consider our CBF \eqref{eq:MA_CBF} and the dynamics defined in (\ref{eq:transition_dynamics}) and (\ref{eq:dynamics_agents}). Based on these, we can compute the following CBC with respect to each other agent $i$ as follows:

\begingroup
\small
\begin{equation}
\begin{aligned}
& CBC^{(i)}(x_t, u_t) = \Big< \frac{f_p(x_t) + g_p(x_t) u_t + d_p(x_t) - f_p^{(i)}(x_t) - d_p^{(i)}(x_t)}{\| f_p(x_t) + g_p(x_t) u_t + d_p(x_t) - f_p^{(i)}(x) - d_p^{(i)}(x_t) \|}, \\ \\
&  f_v(x_t) + g_v(x_t) u_t + d_v(x_t) - f_v^{(i)}(x_t) - d_v^{(i)}(x_t) \Big> + \\ \\
&  \sqrt{a_{max} (\| f_p(x_t) + g_p(x_t) u_t + d_p(x_t) - f_p^{(i)}(x_t) - d_p^{(i)}(x_t) \| - D_s) } ~ + \\
& (\eta - 1) \sqrt{a_{max} (\| \Delta p_t \| - D_s) } + (\eta - 1) \frac{\Delta p_t^T \Delta v_t}{\| \Delta p_t \| } .
\end{aligned}
\label{eq:robust_cbc}
\end{equation}
\endgroup

If we can \textbf{(a)} determine bounds on the dynamic uncertainties, $d$, in \eqref{eq:transition_dynamics} and \eqref{eq:dynamics_agents}, and \textbf{(b)} compute control actions that satisfy $CBC(x,u) \geq 0$ in an online fashion, then we can obtain robust safety guarantees utilizing the multi-agent CBF. Ideally, we could incorporate (\ref{eq:robust_cbc}) into an efficiently solvable program as follows,
\begin{equation}
\begin{aligned}
u = ~ & \underset{u_t \in \mathcal{U}}{\text{argmin}}
& & \| u - u_{des} \|_2 \\
& ~~~~ \text{s.t.}
& &   \underset{d(x_t)}{\text{min}} ~ CBC^{(i)}(x_t, u, d_t) \geq 0 ~~~ \forall ~ i = 1,...,N \\
& & & \text{where} ~ d(x_t) \in \mathcal{D} \\
& & & \| u \|_2 \leq u_{max} ~ .\\
\end{aligned}
\label{eq:optimization}
\end{equation}

\noindent
where $u_{des}$ is any, potentially unsafe, desired control action passed to our CBF (e.g. a linear MPC controller~\cite{singh2017robust}) and $\mathcal{D}$ is our bound on the uncertainty (to be further discussed in Section V). Note that the CBC constraint~\eqref{eq:robust_cbc} in (\ref{eq:optimization}) is clearly not linear nor convex. Therefore, the resulting program is non-convex and cannot be solved at high frequency for adequate safety assurances. However, recall that our system has relative degree 2, which allows us to derive the following bound,
\begin{equation}
CBC(x_t, u_t, d_t) \geq k_c(x_t) - H_1(x_t) d_t - u_t^T H_2(x_t) d_t - H_3(x_t) u_t,
\label{eq:CBC_bound}
\end{equation}
where the definitions of the terms ($k_c, H_1, H_2, H_3$) are given in~\eqref{eq:parameters} in the Appendix. We also move the derivation of the bound \eqref{eq:CBC_bound} to the Appendix due to space limitations. For the rest of the paper, we drop the index $i$ for notational convenience.

The following lemma allows us to utilize CBC bound \eqref{eq:CBC_bound} to obtain safety guarantees \textit{under polytopic uncertainties}.

\medskip
\begin{lemma}
\textit{Suppose the uncertainty in our dynamics $d$ is bounded in the polytope $\{d \in \mathbb{R}^n ~ | ~ Gd \leq g \}$. Then the action, $u$, obtained from solving the following optimization problem (\ref{eq:program_dual_robust}) robustly satisfies the CBC condition (\ref{eq:robust_cbc}) (i.e. renders the set $\mathcal{C}$ forward invariant).}
\begin{equation}
	\begin{aligned}
	& \underset{u \in \mathcal{U}, \xi \in \mathbb{R}_{+}^{4n}}{\text{min}} ~~ \| u - u_{des} \|_2 \\
	& \text{s.t.} ~~~~~~~~~~~~~~~ H_3(x_t) u + \xi g \leq k_c(x_t) \\
	& ~~~~~~~~~~~~~~~~~~ H_1(x_t) + u^T H_2(x_t) = \xi G \\
	& ~~~~~~~~~~~~~~~~~~ \xi \geq 0 \\
	& ~~~~~~~~~~~~~~~~~~ \| u \|_2 \leq u_{max} ~~~ (\textnormal{actuation limits}) ~ ,
	\end{aligned}
	\label{eq:program_dual_robust}
	\end{equation}
\label{lemma:robust}
\end{lemma}
\begin{proof}
	The robust optimization problem (\ref{eq:program_dual_robust}) can be equivalently represented by the following optimization problem (i.e. (\ref{eq:program_dual_robust}) is the dual to (\ref{eq:program_primal_robust}) with no duality gap \cite{Chen2019} where $\xi$ is the dual variable):
	\begin{equation}
	\begin{aligned}
	& \underset{u \in \mathcal{U}}{\text{min}} ~~ \| u - u_{des} \|_2 \\
	& \text{s.t.} ~~~~~~~~~~~~ \forall d \in \{ d \in \mathbb{R}^n ~ | ~ Gd \leq g \} \\
	& ~~~~~~~~~~~~~~~~ H_1(x_t) d + u^T H_2(x_t) d + H_3(x_t) u \leq k_c(x_t) ~ \\
	& ~~~~~~~~~~~~~~~~ \| u \|_2 \leq u_{max} ~~~ (\textnormal{actuation limits}) ~ ,
	\end{aligned}
	\label{eq:program_primal_robust}
	\end{equation}
	where $u$ is the decision vector, $d$ is the uncertainty variable, and $\{d \in \mathbb{R}^n ~ | ~ Gd \leq g\}$ is the uncertainty bound. If the inequality in (\ref{eq:program_primal_robust}) is satisfied such that $H_1 d + u^T H_2 d + H_3 u \leq k_c$, then it follows directly from (\ref{eq:CBC_bound}) that the CBC condition is satisfied for all $d$ in our polytopic uncertainty set. Therefore, the set $\mathcal{C}$ is rendered forward invariant.
\end{proof}

This lemma shows that if we bound $d(x_t)$ in a polytope, we can transform our robust multi-agent CBF (\ref{eq:optimization}) into a quadratic program (\ref{eq:program_dual_robust}), which gives us a computationally efficient way to provide robust guarantees of safety under robot and environment uncertainties. Hence, in the following section, we examine the problem of learning accurate polytopic bounds on the uncertainty $d$ in an online fashion.

\medskip
\section{LEARNING UNCERTAINTY BOUNDS}
\label{sec:5}
In this section, our goal is to learn accurate confidence supports for the uncertainties $d$ and $d^{(i)}$ (for all agents $i$) in an online manner, which will allow us to guarantee safety with high probability. To this end, we utilize Matrix-Variate Gaussian Processes which provide multivariate Gaussian distributions over the uncertainties, $d$ and $d^{(i)}$. 

Using \eqref{eq:transition_dynamics}, we have $d(x_t) = x_{t+1} - f(x_t) - g(x_t) u_t$. A similar relation holds based on \eqref{eq:dynamics_agents} for other agents. Thus, given a sequence of measurements ($x_t, u_t, x_{t+1}$) over a horizon $T$, we compute the uncertain variable, $d_{t-T}, ..., d_{t-1}$ over that horizon. Then, we infer a distribution over the query point, $d_t$ (i.e. next time point), as described in Equation \eqref{eq:GP_inf}.

\medskip
\noindent
\textbf{Learning Kernel Parameters:} Direct application of the MVG \eqref{eq:GP_inf} to our multi-agent setup will be problematic without first training the MVG model hyperparameters. This is easy to see by noting that the covariance, $\Sigma(X) \otimes \Omega$, does not depend on the observed values, $Y$. Furthermore, the coupling between uncertainties, captured by $\Omega$, is completely independent of our online measurements. Instead, much of the uncertainty prediction is baked into the kernel parameters, $\kappa(l, \sigma)$, and matrix $\Omega$. Thus, to obtain accurate estimates of $d$, we must learn MVG model parameters offline from data. In other words, some agents might behave predictably and others might behave more erratically, and hyperparameter optimization is necessary to capture these uncertainty profiles in our Bayesian inference.

Based on the probability density function \eqref{eq:MVGP}, we obtain the negative log-likelihood of a given set of training data $X$,
\begin{equation}
\begin{split}
L(\bfX, & ~ \bfY ; ~ K, \Omega) = -\ln p(\bfX, ~ \bfY; K, \Omega) = \\
& \frac{N n}{2} \ln(2 \pi) + \frac{n}{2} \ln |K| + \frac{N}{2} \ln |\Omega| + \frac{1}{2} \tr[(K)^{-1} Y \Omega^{-1} Y^T]
\end{split}
\end{equation}
which we optimize (over $\Sigma, \Omega$) using Stochastic Gradient Descent (see hyperparameter optimization in Fig.~\ref{fig:block_diagram}) \cite{Crepey2019}. Recall that $N$ denotes the number of training samples in our batch, and $n$ denotes the dimension of the output $Y$ (i.e. $d(x)$). We run the optimization several times with different initializations to decrease our chance of getting stuck in poor local optima. The gradient update expressions are shown in Equation \eqref{eq:gradients} below. Note that we use projected gradient updates for $\Omega$, in order to enforce the condition that $\Omega$ must be positive definite.

\begin{equation}
\begin{split}
    & \frac{dL}{dl} = \frac{n}{2} \tr \Big( K^{-1} \frac{dK}{dl} \Big) + \frac{1}{2} \tr \Big( -K^{-1} \frac{dK}{dl} K^{-1} Y \Omega^{-1} Y^T \Big) \\ \\
    & \frac{dL}{d\sigma} = \frac{n}{2} \tr \Big( K^{-1} \frac{dK}{d\sigma} \Big) + \frac{1}{2} \tr \Big( -K^{-1} \frac{dK}{dl} K^{-1} Y \Omega^{-1} Y^T \Big) \\ \\
    & \frac{dL}{d\Omega} = \frac{N}{2} \Omega^{-1} - \frac{1}{2} \Omega^{-1} Y^T K^{-1} Y \Omega^{-1}
\end{split}
\label{eq:gradients}
\end{equation}

\medskip
\noindent
\textbf{Converting GP Uncertainty to Polytopic Bound:} After learning the kernel parameters, we can obtain the mean, $\mu_d = \hat{M}$, and variance, $\Sigma_d = \hat{\Sigma} \otimes \hat{\Omega}$, from data observed online based on the multivariate Gaussian Process \eqref{eq:GP_inf}. Then, the uncertainties should follow the distribution,
\begin{equation}
(d - \mu_d)^T \Sigma^{-1}_d (d - \mu_d) \sim \chi^2_N ,
\label{eq:chi_dist}
\end{equation}

\noindent
where $\chi^2_N$ represents the chi-squared distribution with N degrees of freedom (equal to dimension of $d$). This allows us to obtain the confidence support,
\begin{equation}
(d - \mu_d)^T \Sigma^{-1}_d (d - \mu_d) \leq k_{\delta} ~~ \textnormal{with probability} ~~ 1-\delta.
\label{eq:chi_interval}
\end{equation}

However, this set defines an \textit{ellipsoid} over $d$ rather than a polytope, which we require for the robust optimization; while we could directly utilize the ellipsoidal constraint, this would not lead to an efficiently solvable QP. To obtain a polytope, we compute the minimum bounding box surrounding the uncertainty ellipsoid. 
\medskip
\begin{lemma}
\textit{Suppose our robot/environment uncertainty can be described by our MVG model (described by the distribution (\ref{eq:chi_dist})). With probability $1-\delta$, the following polytopic bound on the uncertainty $d$ holds:}
\begin{equation}
-\sqrt{k_{\delta} \lambda_i} + \upsilon_i^T \mu_d \leq \upsilon_i^T d \leq \sqrt{k_{\delta} \lambda_i} + \upsilon_i^T \mu_d,
\label{eq:poly_bound}
\end{equation}
\textit{where $\upsilon_i$ and $\lambda_i$ represent the eigenvectors and eigenvalues of $\Sigma_d$, respectively.}
\label{lemma:MVG}
\end{lemma}
\medskip
\begin{proof}
Since $\Sigma_d$ is a positive symmetric covariance matrix, the eigendecomposition of $\Sigma_d$ exists, and we have $\Sigma_d = \Psi^T \Lambda \Psi$, where $\Lambda$ is a diagonal matrix containing positive eigenvalues of $\Sigma_d$ and $\Psi$ is the orthogonal eigenvector matrix. Thus,~\eqref{eq:chi_interval} can be rewritten as follows:
\begin{equation}
\begin{split}
& \Big[ \Psi^T(d - \mu_d) \Big]^T \Lambda^{-1} \Big[ \Psi^T (d - \mu_d) \Big] \leq k_{\delta} ~~ \textnormal{w.p.} ~~ 1-\delta. \\
\end{split}
\end{equation}
 We can bound the left hand side:
\begin{equation}
\begin{split}
& \Big[ \Psi^T(d - \mu_d) \Big]^T \Lambda^{-1} \Big[ \Psi^T (d - \mu_d) \Big] \\
& ~~ = \sum_i \Big[ \upsilon_i^T(d - \mu_d) \Big]^T \lambda_i^{-1} \Big[ \upsilon_i^T (d - \mu_d) \Big] \\
& ~~ \geq \Big[ \upsilon_i^T(d - \mu_d) \Big]^T \lambda_i^{-1} \Big[ \upsilon_i^T (d - \mu_d) \Big] ~~ \forall ~~ i = 1,...,N
\end{split}
\end{equation}
where $\upsilon_i$ represent the eigenvectors of $\Sigma_d$ contained in $\Psi$, and $\lambda_i$ are the eigenvalues of $\Sigma_d$ contained in $\Lambda$. We can then conclude that with probability $1-\delta$, the following relations hold giving us our polytopic bound,
\begin{equation}
\begin{split}
& \Big[ \upsilon_i^T(d - \mu_d) \Big]^T \lambda_i^{-1} \Big[ \upsilon_i^T (d - \mu_d) \Big] \leq k_{\delta} ~~ \forall ~~ i = 1,...,N \\ \\
& -\sqrt{k_{\delta} \lambda_i} + \upsilon_i^T \mu_d \leq \upsilon_i^T d \leq \sqrt{k_{\delta} \lambda_i} + \upsilon_i^T \mu_d ~~ \textnormal{for} ~~ i = 1,...,N
\end{split}
\end{equation}
\end{proof}
\begin{remark}
Recall from Section III.B that in deriving our uncertainty bounds using the Matrix-Variate Gaussian Process, we rely on the assumption that the uncertainties $(d_1, ..., d_N)$ are distributed according to a multivariate Gaussian. While this a strong assumption that may not be valid in general ~\cite{lederer2019uniform}, it can provide a good approximation of agent behavior in many cases. As an alternative, if the uncertainty belongs to Reproducing Kernel Hilbert Space (RKHS) the high confidence bounds developed in~\cite{Safe_mbrl_nips17} could be used.
\oprocend
\end{remark}
\medskip
\noindent
\textbf{High-Confidence Safety Guarantee:} Combining the uncertainty bound on $d$ with our result from Sec.~\ref{secrefiVw} leads us to the main result, summarized in the following Theorem.
\medskip
\begin{theorem}
	\textit{Using the polytopic bounds (\ref{eq:poly_bound}), the control action obtained from the quadratic program (\ref{eq:program_dual_robust}) guarantees robust safety (i.e. collision avoidance between agents) with probability at least $1-\delta$.}
\end{theorem}

\medskip
\begin{proof}
We can represent (\ref{eq:poly_bound}) in the form $\{Gd \leq g \}$; therefore, with probability $1-\delta$, the uncertainty $d$ is contained in the set $\{ d \in \mathbb{R}^n | Gd \leq g \}$ (by Lemma \ref{lemma:MVG}). From Lemma \ref{lemma:robust} and Equation (\ref{eq:CBC_bound}), if we solve the quadratic program (\ref{eq:program_dual_robust}), we are guaranteed that $CBC(x, u, d) \geq k_c - H_1 d - u^T H_2 d - H_3 u \geq 0$ for all $d \in \{ d \in \mathbb{R}^n | Gd \leq g \}$. Therefore, the CBF condition is satisifed with probability $1-\delta$, so safety is guaranteed with probability at least $1-\delta$ (by Definition 1 and the forward invariance property of CBFs \cite{Ames2014}). 
\end{proof}

\medskip
\section{RESULTS}
\label{sec:6}

\subsection{Navigation in Unstructured Environment}
\label{sec:simulnavig}

We test our algorithm in a simulated multi-agent environment in which our robot, with nonlinear dynamics satisfying Assumption \ref{assumption:act}, navigates from a start to goal position while avoiding collisions, in the presence of a random number of other agents (3-12 agents). Each of the other agents has a randomized (unknown) goal that they try to reach. Approximately half of them blindly travel from their start to goal position without accounting for others, while the other half exhibit some collision avoidance behavior through their own control barrier functions (with random CBF parameters). An example simulation instance is shown in Fig. \ref{fig:markov_chain_fig}. See the code (referenced below) for further simulation details/parameters and agent dynamics.

We simulate several instances of the other agents moving and interacting, and use this data for hyperparameter optimization of an MVG model as described in Sec.~\ref{sec:5}. We then equip our robot with the robust CBF described in Sec.~\ref{secrefiVw}, using the optimized MVG for uncertainty prediction. 

By running $1000$ simulated tests in randomized environments, we show that the robust CBF avoids collision in $98.5\%$ of cases (when we set $\delta = 0.05$ [cf.~\eqref{eq:chi_interval}]), performing much better than the nominal multi-agent CBF (cf.~\cite{Borrmann2015}), which avoids collisions in $85.0\%$ of cases. The simulation results are summarized in Table \ref{table:sim_results}. 

\begin{table}[h]
\centering
\renewcommand{\arraystretch}{1.3}
\normalsize
\begin{tabular}{ |P{2.0cm}||P{2.6cm}|P{2.6cm}|}
	\hline
	~  & Robust Multi-Agent CBF & Nominal Multi-Agent CBF \\
	\hline
	\hline
	Collision Rate   & 1.5 \%  & 15.0\% \\
	\hline
	Distance to Collision &  $7.4 \pm 2.3$ & $7.3 \pm 2.1$ \\
	\hline
\end{tabular}
\caption{~Performance statistics for the robust vs nominal multi-agent CBF across 1000 randomized trials. For fair comparison, the robust and nominal CBFs were tested in the same randomized 1000 trials. \textbf{Collision Rate:} Percentage of trials that ended in collision. \textbf{Distance to Collision:} For trials without collision, the robot's margin from collision. The closer the robust CBF is to the nominal CBF, the less conservativeness is introduced by the uncertainty prediction.}
\label{table:sim_results}
\end{table}

Robustness must always come at the cost of performance (e.g. we can reach the goal faster if we do not care about collisions). To investigate the conservativeness of our approach, we looked at the uncertainty predictions of the MVG; Fig. \ref{fig:uncertainty} shows the uncertainty ellipse (over the 4-dimensional disturbance, $d$) projected onto the two velocity dimensions, as well as the true disturbances, $d_v$. We found that ($\approx 97\%, ~ 99\%)$ of disturbances, $d$, were within the $(2,~3) \sigma$ confidence ellipsoid, respectively, in line with the expectations of the MVG model.
Furthermore, the results in Table \ref{table:sim_results} show that the robust CBF only introduces slight conservativeness, as the margin from collision (in instances where the CBF was active) was very similar when utilizing the robust CBF vs. the nominal CBF. This suggests that the MVG model does well at modeling the uncertainties.

\begin{figure}[h]
\begin{center}
\includegraphics[width=0.45\textwidth]
{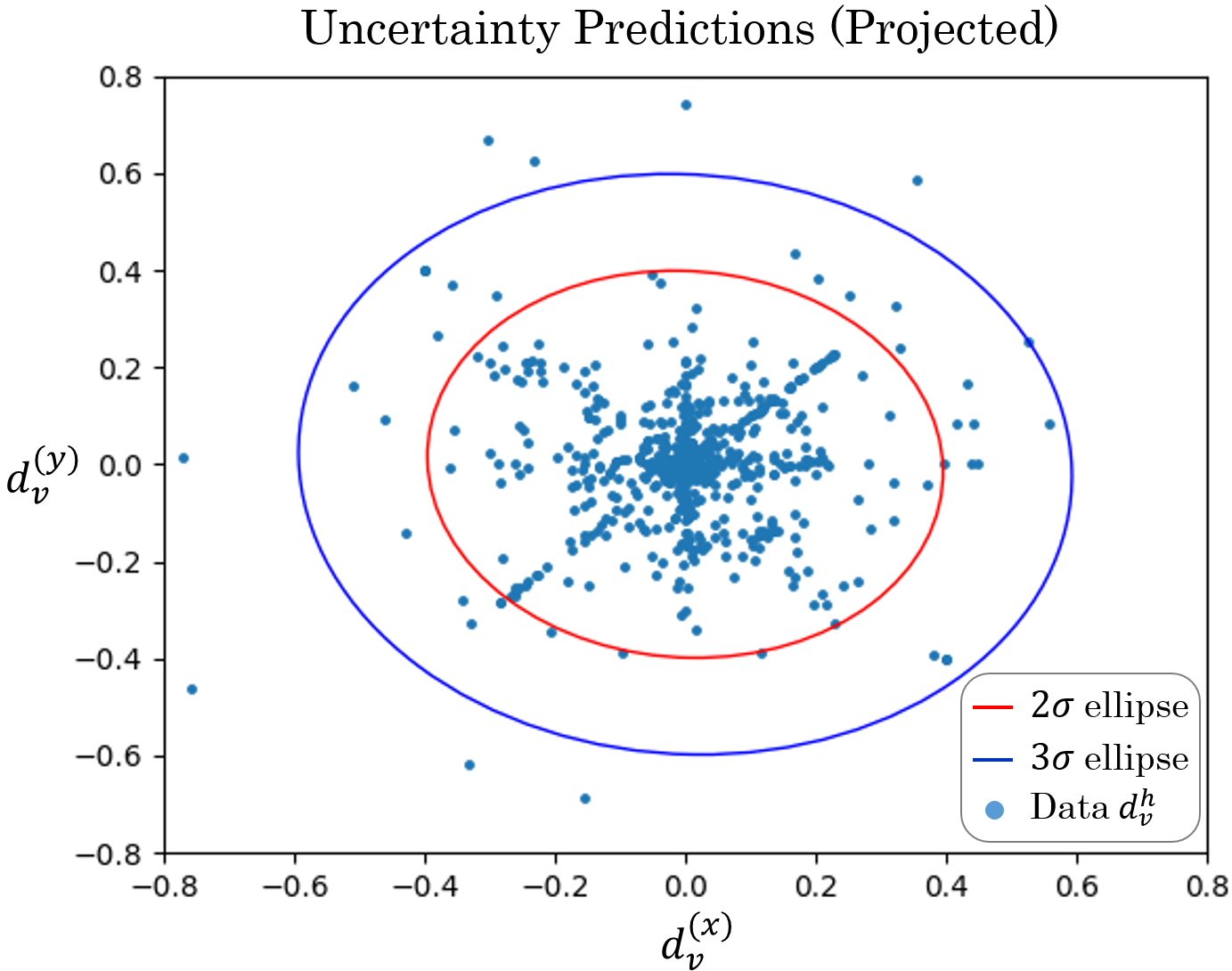}
\end{center}
\caption{~The normalized $2 \sigma$ (red) and $3 \sigma$ (blue) uncertainty ellipsoids over the other agents' dynamics, $d^h$, projected onto the velocity dimensions, $d^h_v$. The true disturbances over 1000 time steps (across different trials) are plotted as the blue dots. We found that the percentage of disturbances within the $2 \sigma$ and $3 \sigma$ uncertainty ellipsoids were consistent with expectations based on the MVG model. Note that the uncertainty ellipse is state-dependent, so we normalize the ellipsoid for each point, $(x^h,d^h_v)$, for fair comparison. }
\label{fig:uncertainty}
\end{figure}

The code for implementing the robust multi-agent CBF in our simulated environment can be found at \texttt{ https://github.com/rcheng805/robust\char`_cbf}. A video of the simulations can be found at \texttt{https://youtu.be/hXg5kZO86Lw}.

\section{CONCLUSION}

Robot navigation in unstructured environments with humans must be safe, but such environments are fraught with uncertainty due to the unpredictability of agents. In this work, we have introduced a robust multi-agent control barrier formulation, which guarantees safety with high probability in the presence of multiple uncontrolled, uncertain agents. We learn uncertainties online for the agents in the environment using Matrix-Variate Gaussian Processes, and design our CBF to be robust to the learned uncertainties. 

Future work will look at learning and designing safe controllers for a larger class of uncertainties, uncaptured by our Matrix Variate GP. Particularly, many agents in the real world exhibit multi-modal uncertainties, and it will be important for us to design safe, robust controllers for such uncertainties.




\bibliographystyle{IEEEtran}
\bibliography{references}

\newpage

\appendices
\section{Parameters of the Control Barrier Condition}
\label{appencbcfo}

Here we define the terms used in the lower bound of the Control Barrier Condition (CBC) in \eqref{eq:CBC_bound}:

\begin{equation}
\begin{split}
& H_1 (1 \times P) = \Big[ -\frac{f_v(x) - f_v^h(x)}{\| f_p(x) - f_p^h(x) \| - \zeta_p(x) - \zeta_p^h(x)}, ~~~ \\ \\
& ~~~~~~~~~~~~~~~~~ -\frac{f_p(x) - f_p^h(x)}{\| f_p(x) - f_p^h(x) \| - \zeta_p(x) - \zeta_p^h(x)}, \\ \\
& ~~~~~~~~~~~~~~~~~ \frac{f_v(x) - f_v^h(x)}{\| f_p(x) - f_p^h(x) \| - \zeta_p(x) - \zeta_p^h(x)}, \\ \\
& ~~~~~~~~~~~~~~~~~ \frac{f_p(x) - f_p^h(x)}{\| f_p(x) - f_p^h(x) \| - \zeta_p(x) - \zeta_p^h(x)}  \Big] \\ \\
& H_2 (M \times P) = \Big[ -\frac{g_v(x)}{\| f_p(x) - f_p^h(x) \| - \zeta_p(x) - \zeta_p^h(x)} ~, ~ 0 ~, \\ \\
& ~~~~~~~~~~~~~~~~~ \frac{g_v(x)}{\| f_p(x) - f_p^h(x) \| - \zeta_p(x) - \zeta_p^h(x)} ~, ~~~ 0 ~~ \Big] \\ \\
& H_3 (1 \times M) = \Big[ -\frac{(f_p(x) - f_p^h(x))^T g_v(x)}{\| f_p(x) - f_p^h(x) \| + \zeta_p(x) + \zeta_p^h(x) } ~ \Big] \\ \\
& k_c ~ = ~ \min \Big( \frac{(f_p(x) - f_p^h(x))^T (f_v(x) - f_v^h(x))}{\| f_p(x) - f_p^h(x) \| ~ \pm ~ \big( \zeta_p(x) + \zeta_p^h(x) \big)} \Big) +  \\ \\
 & ~~~~~~~ \sqrt{a_{max} (\| f_p(x) - f_p^h(x) \| - \zeta_p(x) - \zeta_p^h(x) - D_s) } ~ + \\ \\
 & ~~~~~~~ (\eta - 1) \sqrt{a_{max} (\| \Delta p_t \| - D_s) } ~ + ~ (\eta - 1) \frac{\Delta p_t^T \Delta v_t}{\| \Delta p_t \| } ~ - \\ \\
 & ~~~~~~~ \frac{\zeta_p(x) \zeta_v(x) + \zeta_p(x) \zeta_v^h(x) + \zeta_v^h(x) \zeta_p^h(x) + \zeta_v(x) \zeta_p^h(x)}{\| f_p(x) - f_p^h(x) \| - \zeta_p(x) - \zeta_p^h(x)}
\end{split}
\label{eq:parameters}
\end{equation}

\newpage
\onecolumn
\section{proof of CBC lower bound}
In this section, we prove that the lower bound defined in \eqref{eq:CBC_bound} holds. We begin by expanding out the full CBC condition in  \eqref{eq:robust_cbc}, using our assumption of a relative degree $2$ system, we reach the following expression:
\medskip
\begin{equation*}
\begin{split}
& CBC(x, u, d) = ~~ \frac{(f_p(x) - f_p^h(x))^T (f_v(x) - f_v^h(x))}{\| f_p(x) + d_p(x) - f_p^h(x) - d_p^h(x) \|} + \sqrt{a_{max} (\| f_p(x) + d_p(x) - f_p^h(x) - d_p^h(x) \| - D_s) } ~ + \\
& ~~~~~ (\gamma - 1) \sqrt{a_{max} (\| \Delta p \| - D_s) } + (\gamma - 1) \frac{\Delta p^T \Delta v}{\| \Delta p \| }  ~ + ~ \Big[ ~~ \frac{(f_p(x) - f_p^h(x))^T g_v(x)}{\| f_p(x) + d_p(x) - f_p^h(x) - d_p^h(x) \|} , ~~~ \\
& ~~~~~ \frac{f_v(x) - f_v^h(x)}{\| f_p(x) + d_p(x) - f_p^h(x) - d_p^h(x) \|}, ~~~ \frac{f_p(x) - f_p^h(x)}{\| f_p(x) + d_p(x) - f_p^h(x) - d_p^h(x) \|}, ~~~ -\frac{f_v(x) - f_v^h(x)}{\| f_p(x) + d_p(x) - f_p^h(x) - d_p^h(x) \|}, ~~~ \\
& ~~~~~ -\frac{f_p(x) - f_p^h(x)}{\| f_p(x) + d_p(x) - f_p^h(x) - d_p^h(x) \|} ~~ \Big] ~ \times \Big[ ~ u_R ~ , ~~ d_p ~ , ~~ d_v ~, ~~ d_p^h ~~, ~~ d_v^h ~ \Big]^T ~~ + \\
& ~~~~~ ~ u_R^T~  \Big[ \frac{g_v(x)}{\| f_p(x) + d_p(x) - f_p^h(x) - d_p^h(x) \|} ~, ~~~ 0 ~~, ~~ \frac{-g_v(x)}{\| f_p(x) + d_p(x) - f_p^h(x) - d_p^h(x) \|} ~, ~~~ 0 ~~ \Big] ~ \times \\
& ~~~~~ \Big[ ~ d_p ~ , ~~ d_v ~ , ~~ d_p^h ~~ , ~~ d_v^h ~ \Big]^T  ~~ + ~~ \Big[ ~ 0 ~, ~~ \frac{1}{\| f_p(x) + d_p(x) - f_p^h(x) - d_p^h(x) \|} ~ , ~~\frac{-1}{\| f_p(x) + d_p(x) - f_p^h(x) - d_p^h(x) \|} ~ , ~~ \\
& ~~~~~ \frac{1}{\| f_p(x) + d_p(x) - f_p^h(x) - d_p^h(x) \|} ~ , ~~ \frac{-1}{\| f_p(x) + d_p(x) - f_p^h(x) - d_p^h(x) \|} \Big] \times \Big[ ~ u_R^T u_R ~ , ~~ d_p^T d_v ~, ~~ d_p^T d_v^h ~ , ~~ d_v^h d_p^h ~ , ~~ d_v^T d_p^h ~ \Big]^T 
\end{split}
\end{equation*}

\bigskip
By bounding the positional uncertainty terms $\|d_p(x)\| \leq \zeta_p(x)$ and $\|d_p^h(x)\| \leq \zeta_p^h(x)$, we obtain a lower bound on $CBC(x, u, d)$:
\medskip
\begin{equation*}
\begin{split}
& CBC(x) \geq ~~ \min \Big( \frac{(f_p(x) - f_p^h(x))^T (f_v(x) - f_v^h(x))}{\| f_p(x) - f_p^h(x) \| ~ \pm ~ \big( \zeta_p(x) + \zeta_p^h(x) \big)} \Big) + \sqrt{a_{max} (\| f_p(x) - f_p^h(x) \| - \zeta_p(x) - \zeta_p^h(x) - D_s) } ~ + \\ \\
& ~~~~~ (\gamma - 1) \sqrt{a_{max} (\| \Delta p \| - D_s) } + (\gamma - 1) \frac{\Delta p^T \Delta v}{\| \Delta p \| }  ~ + \Big[ ~ \frac{(f_p(x) - f_p^h(x))^T g_v(x)}{\| f_p(x) - f_p^h(x) \| + \zeta_p(x) + \zeta_p^h(x) } , ~~ \frac{f_v(x) - f_v^h(x)}{\| f_p(x) - f_p^h(x) \| - \zeta_p(x) - \zeta_p^h(x)}, ~~~ \\ \\
& ~~~~~ \frac{f_p(x) - f_p^h(x)}{\| f_p(x) - f_p^h(x) \| - \zeta_p(x) - \zeta_p^h(x)}, ~~ -\frac{f_v(x) - f_v^h(x)}{\| f_p(x) - f_p^h(x) \| - \zeta_p(x) - \zeta_p^h(x)}, ~~ -\frac{f_p(x) - f_p^h(x)}{\| f_p(x) - f_p^h(x) \| - \zeta_p(x) - \zeta_p^h(x)} ~~ \Big] ~ * \\ \\
& ~~~~~ \Big[ ~ u_R ~ , ~ d_p ~ , ~ d_v ~, ~ d_p^h ~, ~ d_v^h ~ \Big]^T  + ~  u_R^T~  \Big[ \frac{g_v(x)}{\| f_p(x) - f_p^h(x) \| - \zeta_p(x) - \zeta_p^h(x)} ~, ~ 0 ~, ~ \frac{-g_v(x)}{\| f_p(x) - f_p^h(x) \| - \zeta_p(x) - \zeta_p^h(x)} ~, ~ 0 ~ \Big] ~ * \\ \\ 
& ~~~~~ \Big[ ~ d_p ~ , ~~ d_v ~ , ~~ d_p^h ~~ , ~~ d_v^h ~ \Big]^T  ~~ - ~~ \frac{\zeta_p(x) \zeta_v(x)}{\| f_p(x) - f_p^h(x) \| - \zeta_p(x) - \zeta_p^h(x)} ~ - \frac{\zeta_p(x) \zeta_v^h(x)}{\| f_p(x) - f_p^h(x) \| - \zeta_p(x) - \zeta_p^h(x)} ~ - ~~ \\ \\ 
& ~~~~~ \frac{\zeta_v^h(x) \zeta_p^h(x)}{\| f_p(x) - f_p^h(x) \| - \zeta_p(x) - \zeta_p^h(x)} ~ - ~ \frac{\zeta_v(x) \zeta_p^h(x)}{\| f_p(x) - f_p^h(x) \| - \zeta_p(x) - \zeta_p^h(x)} 
\end{split}
\end{equation*}

\bigskip
Grouping the terms, this can be written in simplified form using the parameters defined in Appendix \ref{appencbcfo}:
\begin{equation*}
CBC(x, u, d) \geq k_c(x) - H_1(x) \textbf{d} - \textbf{u}^T H_2(x) \textbf{d}  - H_3(x) \textbf{u}
\end{equation*}

\end{document}

%% file: sym.tex
\newcommand{\calM}{{\cal M}}
\newcommand{\calN}{{\cal N}}

\newcommand{\bfu}{\mathbf{u}}

\newcommand{\bfx}{\mathbf{x}}

\newcommand{\bfA}{\mathbf{A}}
\newcommand{\bfB}{\mathbf{B}}

\newcommand{\bfM}{\mathbf{M}}

\newcommand{\bfX}{\mathbf{X}}
\newcommand{\bfY}{\mathbf{Y}}
